\documentclass[runningheads]{llncs}
\usepackage{times}
\usepackage{soul}
\usepackage{url}
\usepackage[hidelinks]{hyperref}
\usepackage[utf8]{inputenc}
\usepackage[small]{caption}
\usepackage{graphicx}
\usepackage{amsmath}
\usepackage{booktabs}
\usepackage{algorithm}
\usepackage{algorithmic}
\usepackage{colortbl}
\usepackage{graphbox}

\usepackage[hidelinks]{hyperref}
\usepackage{url}

\usepackage[l2tabu, orthodox]{nag}

\usepackage{amsmath}	%
\usepackage{amssymb}

\usepackage{graphicx} 		%
\usepackage{caption}
\usepackage{tabularx}
\usepackage{pbox} 			%
\usepackage{multicol, multirow}
\usepackage{tikz}
\usetikzlibrary{fit,arrows,shapes}
\usetikzlibrary{decorations.pathreplacing}

\usepackage
[
]{todonotes}
\usepackage[utf8]{inputenc} %
\usepackage[T1]{fontenc}    %
\usepackage[english]{babel} %
\usepackage{csquotes}		%
\usepackage{hyperref}       %
\usepackage{cleveref}		%
\usepackage{url}            %
\usepackage{nicefrac}       %
\usepackage
[
activate={true,nocompatibility}, %
final,						%
tracking=true,
kerning=true,
spacing=true,
factor=1100,				%
stretch=10,					%
shrink=10					%
]
{microtype}
\usepackage{cancel}			%
\usepackage{mlmacros}
\usepackage{mathrsfs}
\usepackage[nolist,nohyperlinks]{acronym}
\usepackage{subcaption}
\usepackage{pifont}
\newcommand{\cmark}{\ding{51}}%
\newcommand{\xmark}{\ding{55}}%

\begin{acronym}
	\acro{nstats}[NSTATS]{neural statistician}
	\acro{kde}[KDE]{kernel density estimate}
\end{acronym}

\variables[particle,query,aggregation,result,attention,embedding,combination,permutation]{x,q,a,r,w,m,c,\pi}
\probdists[embed,process,p]{\phi,\rho,p}
\MkProbDist{aggregate}{{\alpha}}
\MkProbDist{powerset}{\mathcal{P}}

\newcommand{\LSE}{\operatorname{L\Sigma E}}

\newcommand{\Particle}{\expandafter\MakeUppercase{\particle}}
\newcommand{\population}{\expandafter\mathcal\expandafter\Particle}
\newcommand{\populationtensor}{\expandafter\mathbf\expandafter\Particle}
\newcommand{\particlespace}{\expandafter\mathscr\expandafter\Particle}

\newcommand{\Combination}{\expandafter\MakeUppercase{\combination}}
\newcommand{\combinationtensor}{\expandafter\mathbf\expandafter\Combination}
\newcommand{\Embedding}{\expandafter\MakeUppercase{\embedding}}
\newcommand{\embeddingtensor}{\expandafter\mathbf\expandafter\Embedding}

\title{On Deep Set Learning and the Choice of Aggregations}
\author{
	Maximilian Soelch\and
	Adnan Akhundov\and
	Patrick van der Smagt\and
	Justin Bayer
}

\authorrunning{M.\ Soelch et al.}
\institute{
	argmax.ai, Volkswagen Group Machine Learning Research Lab, Munich, Germany
	\email{m.soelch@argmax.ai}
}

\begin{document}
\maketitle
\begin{abstract}
	Recently, it has been shown that many functions on sets can be represented by sum decompositions.
	These decompositons easily lend themselves to neural approximations, extending the applicability of neural nets to set-valued inputs---Deep Set learning.
	This work investigates a core component of Deep Set architecture: aggregation functions. 
	We suggest and examine alternatives to commonly used aggregation functions, including learnable recurrent aggregation functions.
	Empirically, we show that the Deep Set networks are highly sensitive to the choice of aggregation functions:
	beyond improved performance, we find that  learnable aggregations lower hyper-parameter sensitivity and generalize better to out-of-distribution input size.
\end{abstract}	

\section{Introduction}
Machine learning algorithms make implicit assumptions on the data set encoding.
For instance, feed-forward neural networks assume that data is encoded in a unique vector representation, \eg by one-hot encoding categorical variables.
Yet, many interesting learning tasks revolve around data sets consisting of sets:
depth vision with 3D point clouds, probability distributions represented by finite samples, or operations on unstructured sets of tags \cite{poczos_distribution-free_2013,wang_dynamic_2018,reed_generative_2016}.

Naively, a population\footnote{
	Disambiguating terms like \emph{set} and \emph{sample}, we discuss \emph{data sets} of \emph{populations} of \emph{particles}.}
is embedded by ordering and concatenating particle vectors into a matrix.
While standard neural networks can learn to \emph{imitate} order-invariant behavior, e.g. by random input permutation at each gradient step, such architectures are no true set functions.
Further, they cannot easily handle varying population sizes.
This motivated research into order-invariant neural architectures \cite{vinyals_order_2015-1,guttenberg_permutation-equivariant_2016,edwards_towards_2016,ravanbakhsh_deep_2016}.
From this, the Deep Set framework  emerged, proving that many interesting invariant functions allow for a sum decomposition \cite{zaheer_deep_2017,qi_pointnet_2017,wagstaff_limitations_2019}.
It allows for straightforward application of neural networks that are order-invariant
by design, and can handle varying population sizes.

In this work, we study aggregations---the component of a Deep Set architecture that induces order invariance by mapping a variable-sized population to a fixed-sized description.
After discussing desirable properties and extending the theory around aggregation functions, we suggest multiple alternatives, including learnable recurrent aggregation functions.
Studying them in several experimental settings, we find that the choice of aggregation impacts not only the performance, but also hyper-parameter sensitivity and robustness to varying population sizes.
In the light of these findings, we argue for new evaluation techniques for neural set functions. %
\section{Order-Invariant Deep Architectures}\label{sec:background}

We discuss populations $\population$ of particles $\bparticle$ from a particle space $\particlespace\subset\RR^d$, \ie $\bparticle \in \population$ and $\population \subset \particlespace \subset \RR^d$.
We are further interested in representations $\populationtensor\in \RR^{p\times d}, p = \abs{\population},$ achieved by concatenating the particles of $\population$.
A permutation of the particle axis with a permutation $\permutation$ is denoted by $\populationtensor_\permutation$, \ie $\populationtensor \neq \populationtensor_\permutation$ but $\populationtensor\equiv\population\equiv\populationtensor_\permutation$.
Data sets $\data$ consist of finite populations $\population_i$ of potentially varying size.
\subsection{Invariance, Equivariance, and Decomposition of Invariant Functions}
We study invariant functions according to
\begin{definition}[Invariance]\label{def:invariance} A function $f$ on the power set $\powerset{\particlespace}$ is \emph{order-invariant} if for any permutation $\pi$ and input $\set{\bparticle_1, \dots, \bparticle_N} \in \powerset{\particlespace}$
	\begin{align*}
		f\left(\set{\bparticle_1, \dots, \bparticle_N}\right) = f\left(\set{\bparticle_{\permutation(1)}, \dots, \bparticle_{\permutation(N)}}\right).
	\end{align*}
	
\end{definition}
If it is clear from the context, we will call such functions \emph{invariant}.
When the input is embedded as a matrix, \cref{def:invariance} can be formulated as
$
f(\populationtensor) = f(\populationtensor_\permutation).
$
A related, important notion is that of \emph{equivariant} functions:
\begin{definition}[Equivariance]
	A function $f$ is \emph{equivariant} if input permutation results in equivalent output permutation, \ie for any $\populationtensor$ and $\populationtensor_\permutation$
	\begin{align*}
		f(\populationtensor_\permutation) = (f(\populationtensor))_\permutation.
	\end{align*}	

\end{definition}
In \cite{zaheer_deep_2017}, a defining structural property of order-invariant functions was proven:
\begin{theorem}[Deep Sets, \cite{zaheer_deep_2017}]\label{thm:deepsets}
	A function $f$ on populations $\population$ from \emph{countable} particle space $\particlespace$
	is invariant if and only if there exists a decomposition,
\begin{align*}
	f(\population) = \process{\sum_{\bparticle\in\population}\embed{\bparticle}},
\end{align*}
with appropriate functions $\embed$ and $\process$.
\end{theorem}
We call such functions \emph{sum-decomposable}; this follows \cite{wagstaff_limitations_2019}, where severe pathologies for \emph{uncountable} input spaces are pointed out:
\begin{enumerate}
	\item There exist invariant functions that have no sum decomposition.
	\item There exist sum decompositions that are everywhere-discontinuous.
	\item Even relevant functions such as $\max(\population)$ cannot be \emph{continuously} decomposed when the image space of the embedding $\embed$ is smaller than the population size $|\population|$.
\end{enumerate}
As a consequence they refine \cref{thm:deepsets} to
\begin{theorem}[Uncountable Particle Spaces, \cite{wagstaff_limitations_2019}]\label{thm:continuous}
	A \emph{continuous} function $f$ on \emph{finite} populations $\population$, $|\population| \leq p$, is invariant if and only if it is sum-decomposable via $\RR^p$.
\end{theorem}
That is, for arbitrary $f$, the image space of $\phi$ has to have at least dimension $p$, which is both necessary and sufficient.
More restrictive in scope than \cref{thm:deepsets}, it is more applicable in practice where most function approximators---neural networks, Gaussian processes---are continuous.

\subsection{Deep Sets}
\begin{figure}[tb]
	\centering
	\begin{subfigure}[b]{.57\textwidth}
		\centering
		\begin{tikzpicture}[scale=0.155]
		
		\draw [blue!80, fill=blue!10] (10, 15) rectangle ++(5, 11);
		\draw [blue!80, fill=blue!10] (9, 14) rectangle ++(5, 11);
		\draw [blue!80, fill=blue!10] (8, 13) rectangle ++(5, 11) node [black, pos=0.5] {X};
		
		\draw [->, >=latex, blue!50, semithick] (15, 20.5) -- (17, 20.5);
		
		\draw [blue!80] (17, 26) -- (21, 24) -- (21, 17) -- (17, 15) -- cycle;
		\draw (19, 20.5) node {$\embed$};
		
		\draw [->, >=latex, blue!50, semithick] (21, 20.5) -- (23, 20.5);
		
		\draw [blue!50, fill=orange!10] (22, 15.5) -- (25, 18.5) -- (30, 18.5) -- (27, 15.5) -- cycle;
		\draw [blue!80, fill=blue!10] (25, 18) rectangle ++(4, 8);
		\draw [blue!80, fill=blue!10] (24, 17) rectangle ++(4, 8);
		\draw [blue!80, fill=blue!10] (23, 16) rectangle ++(4, 8) node [black, pos=0.5] {M};
		
		\draw [->, >=latex, blue!50, semithick, rounded corners=1] (25, 15.5) -- (25, 14) -- (28, 14);
		
		\draw [blue!80] (29, 14) circle [radius=1] node [black] {$\oplus$};
		
		\draw [->, >=latex, blue!50, semithick] (29, 15) -- (27, 16.2);
		\draw [->, >=latex, blue!50, semithick] (29, 15) -- (28, 17.2);
		\draw [->, >=latex, blue!50, semithick] (29, 15) -- (29, 18.2);
		
		\draw [->, >=latex, blue!50, semithick] (29, 20.5) -- (31, 20.5);
		
		\draw [blue!80] (32, 20.5) circle [radius=1] node [black] {$\sigma$};
		
		\draw [->, >=latex, blue!50, semithick] (33, 20.5) -- (35, 20.5);
		
		\draw [blue!80] (36, 20.5) circle [radius=1] node [black] {$\oplus$};
		
		\draw [->, >=latex, blue!50, semithick] (37, 20.5) -- (39, 20.5);
		
		\draw [blue!80, fill=blue!10] (39, 17) rectangle ++(3, 7) node [black, pos=0.5] {$\aggregation$};
		
		\draw [->, >=latex, blue!50, semithick] (42, 20.5) -- (44, 20.5);
		
		\draw [blue!80] (44, 24) -- (47, 20.5) -- (44, 17) -- cycle;
		\draw (45, 20.5) node {$\process$};
		
		\draw [->, >=latex, blue!50, semithick] (47, 20.5) -- (49, 20.5);
		
		\draw [blue!80, fill=blue!10] (49, 19) rectangle ++(3, 3) node [black, pos=0.5] {\result};
		\end{tikzpicture}
		\caption{Deep Set Framework.}
		\label{fig:eap}
	\end{subfigure}
	\hfill\hfill
	\begin{subfigure}[b]{.35\textwidth}
		\centering
		\begin{tikzpicture}[scale=0.11]
		
		\draw [blue!80, fill=blue!10] (8, 11) rectangle ++(6, 12);
		\draw [blue!80, fill=blue!10] (7, 9.5) rectangle ++(6, 12);
		\draw [blue!80, fill=blue!10] (6, 8) rectangle ++(6, 12) node [black, pos=0.5] {M};
		
		\draw [blue!50, densely dashed] (17, 7) rectangle ++(26, 16);
		\draw [blue!50, densely dashed] (43, 20) rectangle ++(3, 3);
		\draw (44.5, 21.5) node [black] {$\oplus$};
		
		\draw [blue!80] (20, 20) circle [radius=1.5] node [black] {$\query_1$};
		\draw [blue!80] (27, 20) circle [radius=1.5] node [black] {$\query_2$};
		\draw [blue!80] (40, 20) circle [radius=1.5] node [black] {$\query_T$};
		
		\draw [blue!80] (20, 14.5) circle [radius=1.5] node [black] {$\aggregation_1$};
		\draw [blue!80] (27, 14.5) circle [radius=1.5] node [black] {$\aggregation_2$};
		\draw [blue!80] (40, 14.5) circle [radius=1.5] node [black] {$\aggregation_T$};
		
		\draw [blue!80, fill=blue!10] (18.75, 8.5) rectangle ++(2.5, 2.5);
		\draw [blue!80, fill=blue!10] (25.75, 8.5) rectangle ++(2.5, 2.5);
		\draw [blue!80, fill=blue!10] (38.75, 8.5) rectangle ++(2.5, 2.5);
		
		\draw [blue!80] (20, 4) circle [radius=1.5] node [black] {\aggregation};
		
		\draw [->, >=latex, blue!50, semithick] (14, 17.5) to [out=0] (18.9, 15.5);
		\draw [->, >=latex, blue!50, semithick] (14, 17.5) to (20, 17.5) to [out=0] (25.9, 15.5);
		\draw [->, >=latex, blue!50, semithick] (14, 17.5) to (33, 17.5) to [out=0] (38.9, 15.5);
		
		\draw [->, >=latex, blue!50, semithick] (21.5, 20) -- (25.5, 20);
		\draw [->, >=latex, blue!50, semithick] (28.5, 20) -- (38.5, 20) node[midway, fill=white, inner sep=2.0] {...};
		
		\path [fill=white] (20, 17.5) circle [radius=0.3];
		\draw [->, >=latex, blue!50, semithick] (20, 18.5) -- (20, 16);
		\path [fill=white] (27, 17.5) circle [radius=0.3];
		\draw [->, >=latex, blue!50, semithick] (27, 18.5) -- (27, 16);
		\draw [->, >=latex, blue!50, semithick] (40, 18.5) -- (40, 16);
		
		\path [fill=white] (23.3, 17) circle [radius=0.4];
		\path [fill=white] (23.8, 17.5) circle [radius=0.4];
		\draw [->, >=latex, blue!50, semithick] (21.1, 15.5) -- (25.9, 19);
		\path [fill=white] (30.8, 17.5) circle [radius=0.4];
		\draw [->, >=latex, blue!50, semithick] (28.1, 15.5) -- (32.9, 19);
		\path [fill=white] (36.3, 17.1) circle [radius=0.4];
		\draw [->, >=latex, blue!50, semithick] (34.1, 15.5) -- (38.9, 19);
		
		\path [blue!50, draw] (33.5, 14.5) node {...};
		
		\draw [->, >=latex, blue!50, semithick] (20, 13) -- (20, 11);
		\draw [->, >=latex, blue!50, semithick] (27, 13) -- (27, 11);
		\draw [->, >=latex, blue!50, semithick] (40, 13) -- (40, 11);
		
		\draw [->, >=latex, blue!50, semithick] (25.75, 9.75) -- (21.25, 9.75);
		\draw [->, >=latex, blue!50, semithick] (38.75, 9.75) -- (28.25, 9.75) node[midway, fill=white, inner sep=2.0] {...};
		
		\draw [->, >=latex, blue!50, semithick] (20, 8.5) -- (20, 5.5);
		\end{tikzpicture}
		\caption{Recurrent Aggregation.}
		\label{fig:reap}
	\end{subfigure}
	\caption{\emph{Left}: Deep Set architecture,
		\cref{eq:deep_sets_aggregate,eq:deep_sets_embed,eq:deep_sets_combine,eq:deep_sets_process},
		 with a single equivariant layer, \cref{eq:equivariant}.
		Aggregation functions are depicted by $\oplus$.
		\emph{Right}: Recurrent aggregation function, \cref{eq:recurrent_aggregation_1,eq:recurrent_aggregation_2,eq:recurrent_aggregation_3,eq:recurrent_aggregation_4,eq:recurrent_aggregation_5}.
		Queries to memory are produced in a forward pass, responses aggregated in a backward pass.
	}
\end{figure}
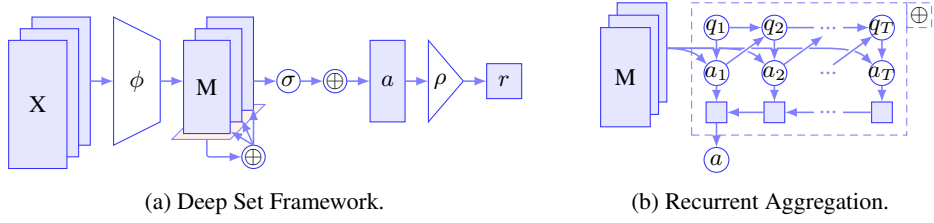
A generic invariant neural architecture emerges from \cref{thm:deepsets,thm:continuous} by using  neural networks for $\process$ and $\embed$, respectively.
In practice, to allow for higher-level particle interaction during the embedding $\embed$, \emph{equivariant} neural layers are introduced \cite{zaheer_deep_2017},
\begin{align}
	\operatorname{equivariant}(\populationtensor) = \sigma\mleft(\populationtensor - \mathbf{1}\aggregate{\populationtensor}\mright),\label{eq:equivariant}
\end{align}
where $\sigma(\cdot)$ denotes a per-particle feed-forward layer, and $\alpha(\cdot)$ denotes an aggregation.
Aggregations---our object of study---induce invariance by mapping a population to a fixed-size description, typically \eg sum, mean, or $\max$.
The full architecture is
\begin{align}
	\bembedding_i &= \operatorname{embed}(\bparticle_i)\label{eq:deep_sets_embed}\\
	\combinationtensor &= \operatorname{combine}(\embeddingtensor)&\mleft(\embeddingtensor = \mleft[\bembedding_i^\top\mright]\label{eq:deep_sets_combine}\mright)\\
	\baggregation &= \operatorname{aggregate}(\combinationtensor)&\label{eq:deep_sets_aggregate}\\
	\bresult &= \operatorname{process}(\baggregation)\label{eq:deep_sets_process},
\end{align}
with $\embed$ implemented by a per-particle embedding followed by an equivariant combination function consisting of equivariant layers.
Summation is replaced by a generic aggregation operation.
In \cite{zaheer_deep_2017,qi_pointnet_2017}, the $\max$ operation is suggested as an alternative summation.
Lastly, $\process$ can be implemented by arbitrary functions, since the aggregation in \cref{eq:deep_sets_aggregate} is already invariant.
This framework is depicted in \cref{fig:eap}.

\subsection{Order Matters}
\label{sub:order_matters}
Recurrent neural networks can handle set-valued input by feeding one particle at a time.
However, it has been shown that the result is sensitive to order, and an invariant \emph{read-process-write} architecture has been suggested as a remedy \cite{vinyals_order_2015-1}:
\eq{
	\bquery_t &= \operatorname{LSTM}(\bquery\tm,\baggregation\tm)\numberthis\label{eq:order_matters_1}\\
	\hat{\attention}_{i,t} &= \operatorname{attention}(\bembedding_i, \bquery_t)&\left(=\bembedding_i^\top\bquery_t\right)\numberthis\label{eq:order_matters_2}\\
	\battention_t &= \operatorname{softmax}(\hat{\battention}_t)\numberthis\label{eq:order_matters_3}\\
	\baggregation_t &= \sum \attention_{i,t}\bembedding_i\numberthis\label{eq:order_matters_4}\\
	\baggregation &= \baggregation_T\numberthis\label{eq:order_matters_5}
}
An embedded memory  is queried. 
The invariant result $\baggregation_t$ is iteratively used to refine subsequent queries with an LSTM \cite{hochreiter_long_1997}.
It is not obvious how to cast the recurrent structure into the setting of \cref{eq:deep_sets_aggregate,eq:deep_sets_combine,eq:deep_sets_process,eq:deep_sets_embed,thm:deepsets,thm:continuous}.
To the best of our knowledge, this model has only been discussed in its sequence-to-sequence context.
We will revisit and refine this architecture in \cref{sub:learnableaggr}.

\subsection{Further Related Work}
Several papers introduce and discuss a Deep Set framework for dealing with set-valued inputs \cite{qi_pointnet_2017,zaheer_deep_2017}.
A driving force behind research into order-invariant neural networks are point clouds \cite{qi_pointnet++_2017,qi_frustum_2017,qi_pointnet_2017}, where such architectures are used to perform classification and semantic segmentation of objects and scenes represented as point clouds in $\mathbb{R}^{3}$.
It is further shown that a $\max$ decomposition allows for arbitrarily close approximation \cite{qi_pointnet_2017}.

Generative models of sets have been investigated:
in an extension of variational auto-encoders \cite{kingma_auto-encoding_2013,rezende_stochastic_2014},  the inference of latent population statistics resembles a Deep Sets architecture \cite{edwards_towards_2016}.
Generative models of point clouds are proposed by \cite{achlioptas_learning_2018} and \cite{yi_gspn_2018}.

Permutation-invariant neural networks have been used for predicting dynamics of interacting objects \cite{guttenberg_permutation-equivariant_2016}. The authors propose to embed the individual object positions in pairs using a feed-forward neural network.
Similar pairwise approaches have been investigated by \cite{NIPS2014_5545,chang_compositional_2016}, and applied to relational reasoning in \cite{santoro_simple_2017}.

Weighted averages based on attention have been proposed and applied to multi-instance learning \cite{ilse_attention-based_2018}.
Several works have focused on higher-order particle interaction, suggesting computationally efficient approximations of Janossy pooling \cite{murphy_janossy_2018}, or propose set attention blocks as an alternative to equivariant layers \cite{lee_set_2018}. %
\section{The Choice of Aggregation}\label{sec:methods}

The invariance of the Deep Set architecture emerges from invariance of the \emph{aggregation function}---\cref{eq:deep_sets_aggregate}.
\Cref{thm:deepsets} theoretically justifies summing the embeddings $\embed{\bparticle_i}$.
In practice, mean or max-pooling operations are used. 
Equally simple and invariant, they are numerically favorable for varying population sizes, controlling input magnitude to downstream layers.
This section discusses alternatives and their properties.

\subsection{Alternative Aggregations}\label{sub:alternative_aggregations}
We start by justifying alternative choices with an extension of \cref{thm:deepsets,thm:continuous}:
\begin{corollary}[Sum Isomorphism]\label{cor:sum_isomorphism}
	\Cref{thm:deepsets,thm:continuous} can be extended to aggregations of the form $\aggregate_g = {g} \circ {\sum} \circ {g^{-1}}$, \ie summations in an isomorphic space.
\end{corollary}
\begin{proof}
	From ${\process} \circ {\sum} \circ {\embed} = ({\process} \circ {g^{-1}}) \circ {g} \circ {\sum} \circ {g^{-1}} \circ ({g} \circ {\embed})$, sum decompositions can be constructed from $\aggregate_g$-decompositions and vice versa.
\end{proof}
This class includes, \eg, mean (with $g((\particle_1, \dots, \particle_{n+1})) = (\particle_1,\dots,\particle_n)/\particle_{n+1}$ and $g^{-1}(\bparticle) = (\bparticle^\top, 1)^\top$) and $\operatorname{logsumexp}$ ($\operatorname{L\Sigma E}$) (with $g= \ln$).
\begin{figure}[tb]
	{	\captionsetup{justification=raggedright,singlelinecheck=false}
		\centering
		\begin{subfigure}[b]{.15\textwidth}
			\centering
			\includegraphics[width=\linewidth]{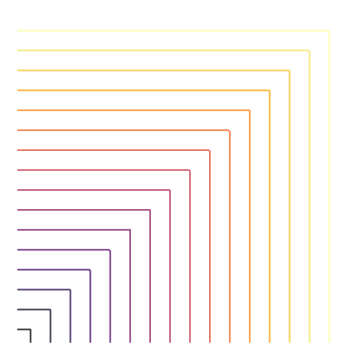}
			
			\caption{$\max$ on $[-10,10]^2$}
			\label{fig:max}
		\end{subfigure}
		\hfill
		\begin{subfigure}[b]{.15\textwidth}
			\centering
			\includegraphics[width=\linewidth]{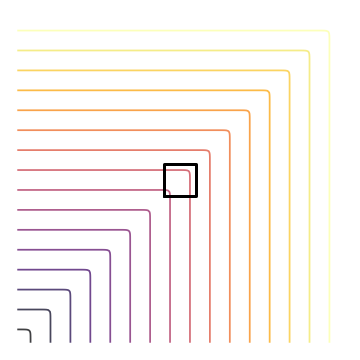}
			\caption{$\LSE$ on ${[-100,100]^2}$}
			\label{fig:lse1}
		\end{subfigure}
		\hfill
		\begin{subfigure}[b]{.15\textwidth}
			\centering
			\includegraphics[width=\linewidth]{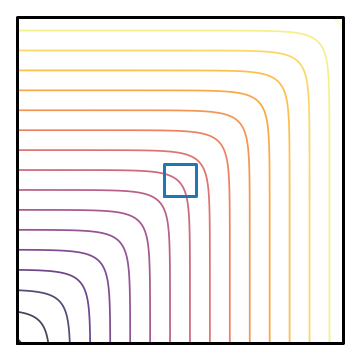}
			\caption{$\LSE$ on $ {[-10,10]^2}$}
			\label{fig:lse2}
		\end{subfigure}
		\hfill
		\begin{subfigure}[b]{.15\textwidth}
			\centering
			\includegraphics[width=\linewidth]{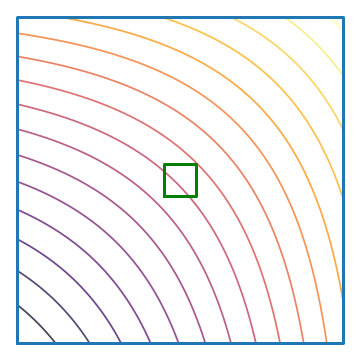}
			\caption{$\LSE$ on $ {[-1,1]^2}$}
			\label{fig:lse3}
		\end{subfigure}
		\hfill
		\begin{subfigure}[b]{.15\textwidth}
			\centering
			\includegraphics[width=\linewidth]{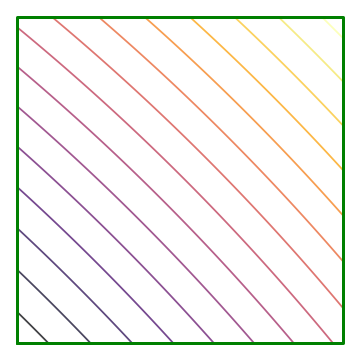}
			\caption{$\LSE$ on $ {[-.1,.1]^2}$}
			\label{fig:lse4}
		\end{subfigure}
		\hfill
		\begin{subfigure}[b]{.15\textwidth}
			\centering
			\includegraphics[width=\linewidth]{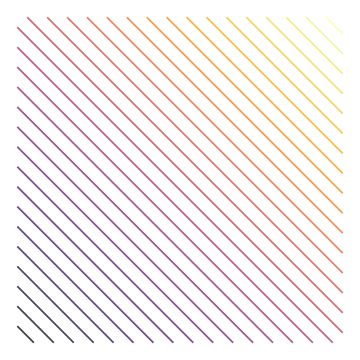}
			\caption{$\operatorname{sum}$ on $ {[-10,10]^2}$}
			\label{fig:sum}
	\end{subfigure}}
	\caption{
		Contour plots for $\max$ (left), sum (right), and $\operatorname{logsumexp}$ ($\LSE$) on two inputs.
		For large ranges, $\LSE$ acts like $\max$, shifting towards $\operatorname{sum}$ with decreasing input range.
		Matching square boxes indicate zoom between plots.
		Plots (a), (c), and (f) on range ${[-10,10]^2}$ share contour levels.
	}
	\label{fig:maxsumlogsumexp}
\end{figure} 
In that light, there is an interesting case to be made for $\LSE$:
depending on the input magnitudes, $\LSE$ can behave akin to $\max$ (cf.\ \cref{fig:max,fig:lse1,fig:lse2}) or like a linear function akin to summation (cf.\ \cref{fig:sum,fig:lse3,fig:lse4}).
Operating in log space, $\LSE$ further exhibits \emph{diminishing returns}:
$N$ identical scalar particles $\particle_i$ yield $\LSE(\set{\particle_i}) = \ln(N) + \particle_1$.
The larger $N$, the smaller the output change from additional particles.
Beyond making $\LSE$ a numerically useful aggregation, diminishing returns are a desirable property from a statistical perspective, where we would like to have asymptotically consistent results.

\paragraph{Divide and Conquer}\label{sub:d_and_c}
Commutative and associative binary operations like addition and multiplication yield invariant aggregations.
Widening this perspective, we see that divide-and-conquer style operations yield invariant aggregations:
order invariance is equivalent to conquering being invariant to division.
Examples beyond the previously mentioned operations are logical operators such as any or all, but also sorting (generalizing $\max$ and $\min$, and any percentile, \eg median).
While impractical for typical first-order optimization, we note that aggregations can be of very sophisticated nature.

\subsection{Learnable Aggregation Functions}
In \cite{zaheer_deep_2017}, cf.\ \cref{eq:deep_sets_embed,eq:deep_sets_combine,eq:deep_sets_aggregate,eq:deep_sets_process}, the aggregation is the only non-learnable component.
We will now investigate ways to render the aggregations learnable.
In \cref{sub:order_matters}, we have seen that due to the structure of \cref{thm:deepsets}, recurrent architectures as suggested by \cite{vinyals_order_2015-1} had been overlooked as it is not straightforward to cast them into the Deep Sets framework.
Inspired by the read-process-write architecture,
we suggest \emph{recurrent aggregations}:
\begin{definition}[Recurrent and Query Aggregation]\label{def:recurrent_aggregation}
	A \emph{recurrent aggregation} is a function $f(\population) = \baggregation$ that can be written recursively as:
	\eq{
		\bquery_t &= \operatorname{query}(\bquery\tm,\baggregation\tm)\numberthis\label{eq:recurrent_aggregation_1}\\
		\hat{\attention}_{i,t} &= \operatorname{attention}(\bembedding_i, \bquery_t)\numberthis\label{eq:recurrent_aggregation_2}\\
		\battention_t &= \operatorname{normalize}(\hat{\battention}_{t})\numberthis\label{eq:recurrent_aggregation_3}\\
		\baggregation_t &= \operatorname{reduce}\left(\left\{\attention_{i,t}\bembedding_i\numberthis\label{eq:recurrent_aggregation_4}\right\}\right)\\
		\baggregation &= g\mleft(\baggregation\Ts\mright),\numberthis\label{eq:recurrent_aggregation_5}
	}
	where $	\bembedding_i = \embed{\bparticle_i}$ is an embedding of the input population $\set{\bparticle_i}$ and $\bquery_1$ is a constant.
	We further call the special case $T=1$ (\ie a single query $\bquery \equiv \bquery_1$) a \emph{query aggregation}.
\end{definition}
As long as $\operatorname{reduce}$ is invariant and $\operatorname{normalize}$ is equivariant, recurrent and query aggregations are invariant.
This architectural block is depicted in \cref{fig:reap}.

Building upon \cref{eq:order_matters_1,eq:order_matters_2,eq:order_matters_3,eq:order_matters_4,eq:order_matters_5}, recurrent aggregations introduce two modifications:
firstly, we replace a weighted sum by a general weighted aggregation---giving us a rich combinatorial toolbox on the basis of simple invariant functions such as those mentioned in \cref{sub:d_and_c}.
Secondly, we add post-processing of the step-wise results $\baggregation\Ts$.
In practice, we use another recurrent network layer that processes $\baggregation\Ts$ in reversed order.
Without this modification, later queries tend to be more important, as their result is not as easily forgotten by the forward recurrence.
The backward processing reverses this effect, so that the first queries tend to be more important, and the overall architecture is more robust to common fallacies of recurrent architectures, in particular unstable gradients.

Observing \cref{eq:recurrent_aggregation_4}, we note that our learnable aggregation functions wrap around the previously discussed simpler non-learnable aggregations.
A major benefit is that the inputs are weighted---sum becomes weighted average, for instance.
This also allows the model to effectively exploit non-linearities as discussed with $\LSE$ (cf.\ \cref{fig:maxsumlogsumexp}).

\subsection{A Note on Universal Approximation}\label{sub:learnableaggr}
The key promise of \emph{universal} approximation is that a family of approximators (\eg neural nets, or neural sum decompositions) is dense within a wider family of interesting functions \cite{kolmogorov_representation_1957,hecht-nielsen_theory_1989,hornik_multilayer_1989}.
The universality granted by \cref{thm:deepsets,thm:continuous}, through constructive proofs, hinges on sum aggregation.
\Cref{cor:sum_isomorphism} grants flexibility, but does not apply to arbitrary aggregations, like $\max$ or the suggested learnable aggregations.
(Note that $\max$ allows for arbitrary approximation \cite{qi_pointnet_2017}.)
It remains open to what extent the sum can be replaced.
As such, the suggested architectures might not grant universal approximators.
As we will see in \cref{sec:experiments}, however, they provide useful inductive biases in practical settings, much like feed-forward neural nets are usually replaced with architectures targeted towards the task.
It is worth noting that the embedding dimension constraint of \cref{thm:continuous} is rarely met, trading theoretical guarantees for test-time performance.
\section{Experiments}\label{sec:experiments}
We consider three simple aggregations: mean (or weighted sum), $\max$, and $\LSE$.
These are used in equivariant layers and final aggregations, and may be be wrapped into a recurrent aggregation.
This combinatorially large space of configurations is tested in four experiments described in the following sections.

\subsection{Mininmal Enclosing Circle}
\begin{figure}
	\renewcommand{\arraystretch}{.75}
	\centering
	\hfill
	\begin{minipage}{0.3\textwidth}
		\centering
		\includegraphics[width=.8\columnwidth]{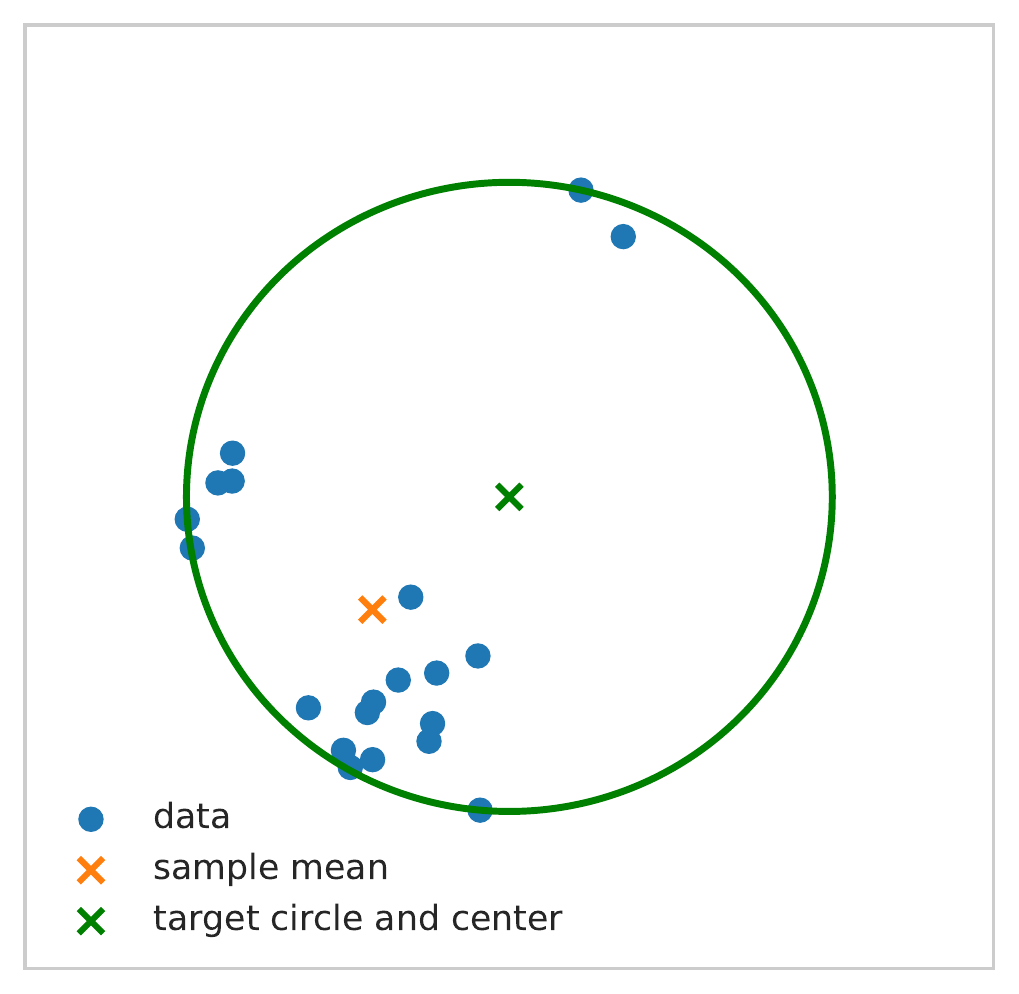}
		\caption{Minimal enclosing example population}\label{fig:mincircle}
	\end{minipage}
\hfill
	\begin{minipage}{0.68\textwidth}
		\centering
		\captionsetup{type=table} %
		\caption{Minimal enclosing circle results.}
		\label{tab:circle}
		\begin{tabular}{p{2cm}*{4}{p{1.cm}}}
			\toprule
			{recurrent \newline equiv./aggr.}& best MSE & radius MSE & center MSE& median best MSE\\\midrule
			\xmark\ / \xmark & 0.71 & 0.06 & 0.66 & 1.57\\
			\xmark\ / \cmark & 1.02 & 0.14 & 0.88 & 1.30\\
			\cmark\ / \xmark & 0.54 & 0.08 & 0.47 &0.87\\
			\cmark\ / \cmark & 0.42 & 0.09 & 0.33 & 0.58\\
			\bottomrule
		\end{tabular}
	\end{minipage}
\end{figure}

In this supervised experiment, we are trying to predict the  minimal enclosing circle of a population of size $20$ from a Gaussian mixture model (GMM).
A sample population with target circle is depicted in \cref{fig:mincircle}.
The sample mean does not approximate the center of the minimal enclosing circle well, and the correct solution is defined by at least three particles.
The models are trained by minimizing the mean squared error (MSE) towards the center and radius of the true circle (computable in linear time \cite{welzl_smallest_1991}).

Results are given in \cref{tab:circle}.
Each row shows the best result out of 180 runs (20 runs for each of the 9 combinations of aggregations).
We can see that both recurrent equivariant layers and recurrent aggregations improve the performance, with equivariant layers granting the larger performance boost.
The challenge lies mostly in a better approximation of the center.

The top row indicates that an entirely non-recurrent model performs better than its counterpart with recurrent aggregation (second row).
To test for a performance outlier, we compute a bootstrap estimate of the expected peak performance when only performing 20 experiments:
we subsample all available experiments (with replacement) into several sets of 20 experiments, recording the best performance in each batch.
The last column in \cref{tab:circle} reports the median of these best batch performances.
The result shows increased robustness to hyper-parameters, despite having more hyper-parameters.

\subsection{GMM Mixture Weights} \label{sub:mixture_weights}
\begin{figure}[tb]
	\centering
	\begin{subfigure}[b]{\textwidth}
		\centering
		\includegraphics[width=.9\textwidth]{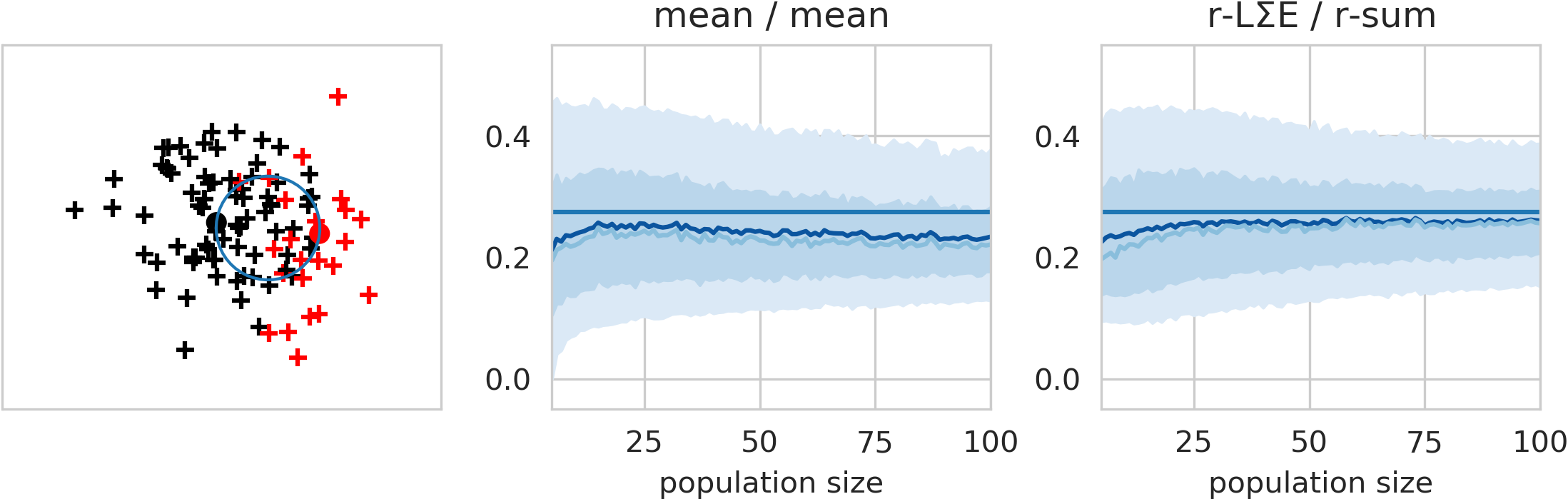}
		\caption{
			\emph{Left}: Example population.
			\emph{Middle and Right}: Estimator development for increasing populations size for a non-learnable and a learnable model, with 50\% and 90\% empirical confidence intervals.
		}
		\label{fig:mw}
	\end{subfigure}\newline
	\begin{subfigure}[b]{\textwidth}
		\centering
		\includegraphics[width=.9\textwidth]{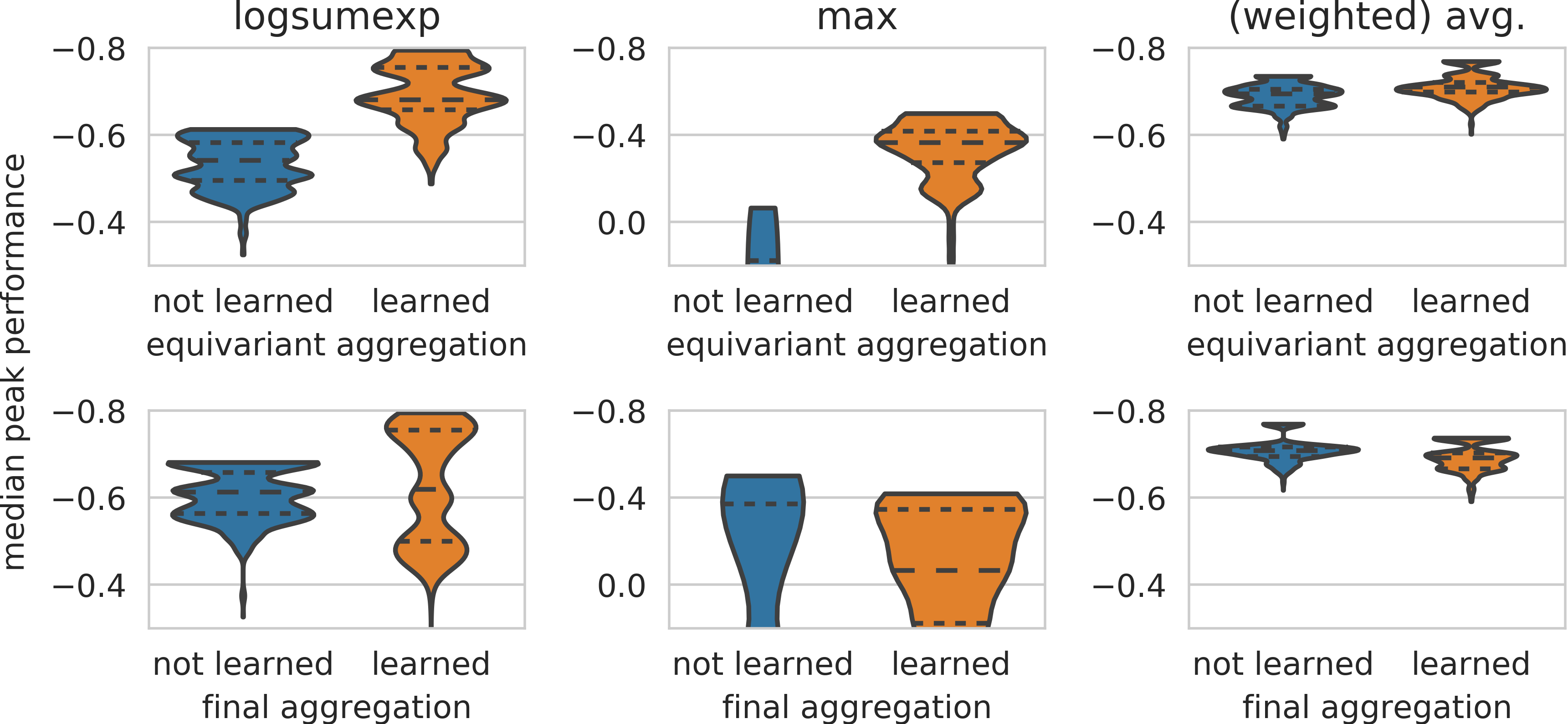}
		\caption{Robustness analysis. Metric is the score ratio of the true mixture weight under a neural model compared to expectation maximization (negative sign indicates EM is outperformed; the more negative, the better). Each violin shows the peak performance distribution for batches of 5 experiments. Top row: equivariant layer aggregations. Bottom row: final aggregations.
		}
		\label{fig:mwviolin}
	\end{subfigure}
	\caption{Results for the Gaussian mixture model mixture weights experiment.}
	\label{fig:mixture_weights}
\end{figure}
In this experiment, our goal is to estimate the mixture weights of a Gaussian mixture model directly from  particles.
The GMM populations of size $100$ in our data set are sampled as follows:
each mixture consists of two components;
the mixture weights are sampled from $[.05, .95]$;
the means span a diameter of the unit circle, their position is drawn uniformly at random;
component variances are a fixed to the same diagonal value such that the clusters are not linearly separable.
An example population is shown in \cref{fig:mw}.
The model outputs concentrations $a$ and $b$ of a Beta distribution. 
We train to maximize the log-likelihood of the smaller ground truth weight under this Beta distribution.
At training time, for every gradient step the batch population size $N$ is chosen randomly, with $\p{N=n}\propto n$.
In \cref{fig:mw}, we show how an estimator based on the learned model behaves with growing population size.

We were again interested in the robustness of the models.
We compare to expectation maximization (EM)---the classic estimation technique for mixture weights---as a baseline by gathering 100 estimates each from EM and the model for each population size by subsampling (with replacement) the original population.
Then we compare the likelihood of the true weight under a kernel density estimate (KDE) of these estimates.
The final metric is the log ratio of the scores under the two KDEs.
Then, as in the previous section, we compute the peak performance for batches of 5 experiments in order to see which configurations of models consistently perform well.

The results of this analysis are shown in \cref{fig:mwviolin}.
The top row indicates that learnable equivariant layers lead to a significant performance boost across all reduction operations.
Note that the y-axis is in log scale, indicating \emph{multiples} of improvements over the EM baseline.
We note that $\LSE$ benefits most drastically from learnable inputs.
Notably, the middle column, which depicts $\max$-type aggregations, indicates that this type of aggregation significantly falls behind the alternatives.
Notice that we had to scale the y-axes to even show the violins, and that a significant amount of \emph{peak} performances perform \emph{worse} than EM (indicated by sign flip of the metric).

\subsection{Point Clouds}\label{sub:point_clouds}
\begin{table}[tb]
	\caption{Test set accuracy on ModelNet40 classification.}
	\label{tab:accuracypc}
	\centering
	\begin{tabular}{*{11}{p{.075\textwidth}}}
		\toprule
		&\multicolumn{10}{c}{Equivariant layer type \& aggregation type}\\
		\vfill$\abs{\population}$&  $\max$\newline $\max$ &  $\max$\newline r-$\LSE$ &  $\max$\newline r-$\operatorname{sum}$ &  $\max$\newline q-$\max$ &  $\max$\newline q-$\operatorname{sum}$ &  r-$\operatorname{sum}$\newline r-$\operatorname{sum}$ &  $\max$\newline r-$\max$ &  r-$\max$\newline r-$\max$ &  r-$\LSE$\newline r-$\LSE$ &  q-sum\newline q-sum \\
		
		\midrule
		1000      &     87.3 &       85.8 &       85.7 &       83.8 &       83.5 &         82.0 &       81.7 &         81.2 &         78.0 &         77.5 \\
		100       &     66.5 &       75.3 &       73.0 &       69.5 &       68.4 &         71.9 &       45.3 &         22.0 &         64.0 &         60.3 \\
		50        &     47.0 &       62.8 &       58.4 &       52.4 &       51.3 &         61.0 &       35.5 &         14.6 &         51.9 &         46.8 \\
		\bottomrule
	\end{tabular}
\end{table}
The previous experiment extensively tested the effect of aggregations in controlled scenarios.
To test the effect of aggregations on a more realistic data set, we tackle classification of point clouds derived from the ModelNet40 benchmark data set \cite{zhirong_wu_3d_2015}.
The data set consists of CAD models describing the surfaces of objects from 40 classes.
We sample point cloud populations uniformly from the surface.
The training is performed on 1000 particles.
For this experiment, we fixed all hyper-parameters---including optimizer parameters and learning rate schedules---as described in \cite{zaheer_deep_2017}, and only exchanged the aggregation functions in the equivariant layers and the final aggregation.

The results for the 10 best configurations are summarized in \cref{tab:accuracypc}.
The original model ($\max$/$\max$ column) performs best in the training scenario ($\abs{\population}=1000$, first row)---as expected on hyper-parameters that were optimized for the model.
Otherwise, learnable final aggregations outperform all non-learnable aggregations.
We further observe that $\max$-type aggregations in equivariant layers seem crucial for good final performance.
This contrasts the findings from \cref{sub:mixture_weights}.
We believe this to be a result of either (i) the hyper-parameters being optimized for $\max$-type equivariant layers, or (ii) the classification task (as opposed to a regression task), favoring $\max$-normalized embeddings that amplify discriminative features.

The second and third row highlight an insufficiently investigated problem with invariant neural architectures:
the top-performing model overfits to the training population size.
Despite sharing all hyper-parameters except the aggregations, the test scenarios with fewer particles show that learnable aggregation functions generalize favorably.
Compare the first two columns:
both drops for the original model are comparable to the total drop for the learnable model.

\subsection{Spatial Attention}\label{sub:spatial_attention}
\begin{figure}[tb]
	\captionsetup[subfigure]{position=b}
	\centering\hspace{.15cm}
	\subcaptionbox{
		Spatial attention example.
		Each pane shows multiple test time bounding box samples for 5, 20, 200, 1000 particles. 
		\label{fig:sqair_data}}{
		\centering
		\begin{tabular}{cc}
			\includegraphics[width=.15\linewidth]{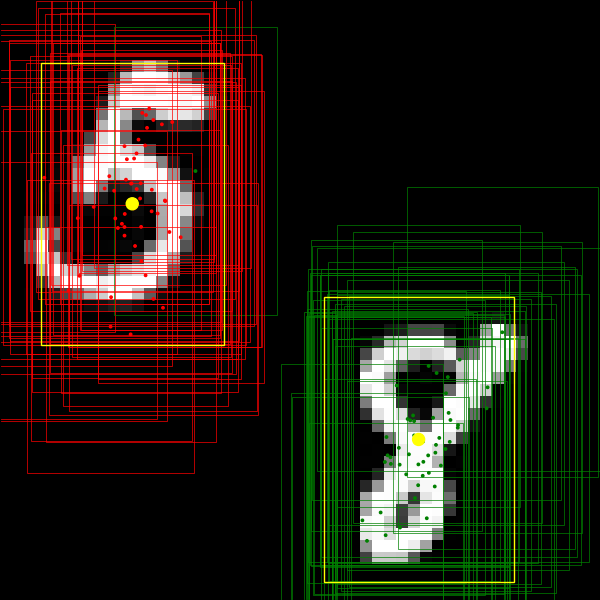}&
			\includegraphics[width=.15\linewidth]{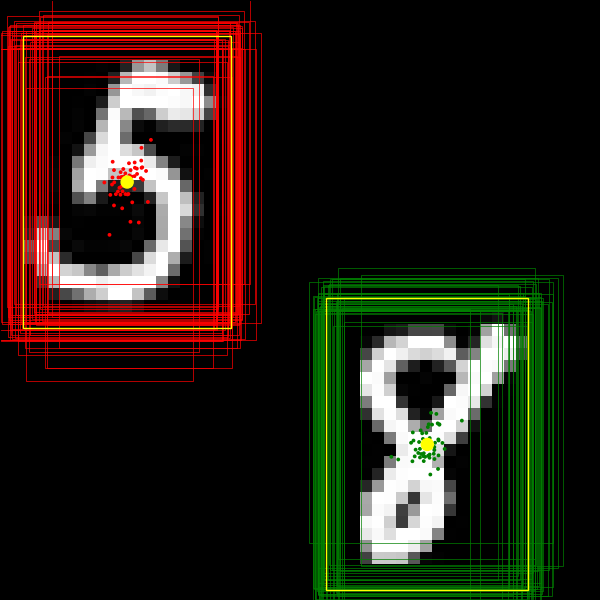}\\
			\includegraphics[width=.15\linewidth]{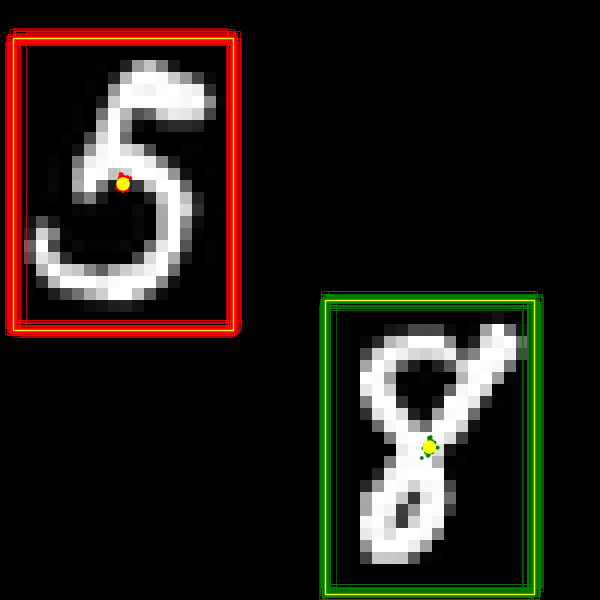}&
			\includegraphics[width=.15\linewidth]{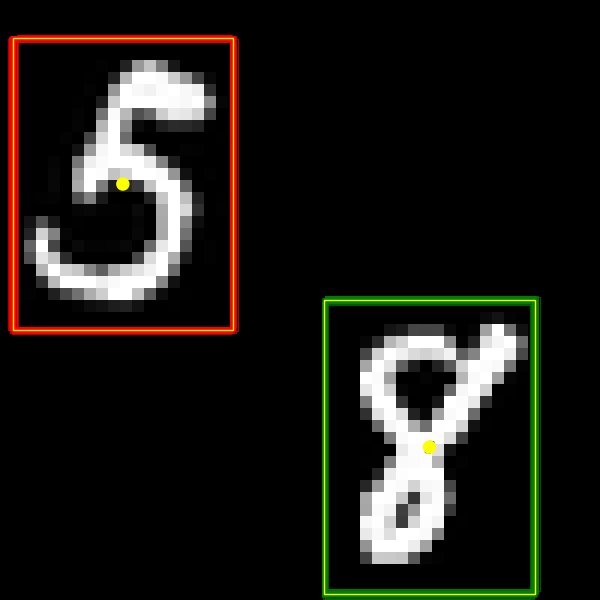}\\
		\end{tabular}
	}\hfill
	\subcaptionbox{
		Test-time evidence lower bound values against various population sizes.
		Dashed vertical line: training population size.
		Dashed horizontal line: best baseline model.
		 \label{fig:air_elbo}}{\centering	\includegraphics[width=.55\linewidth]{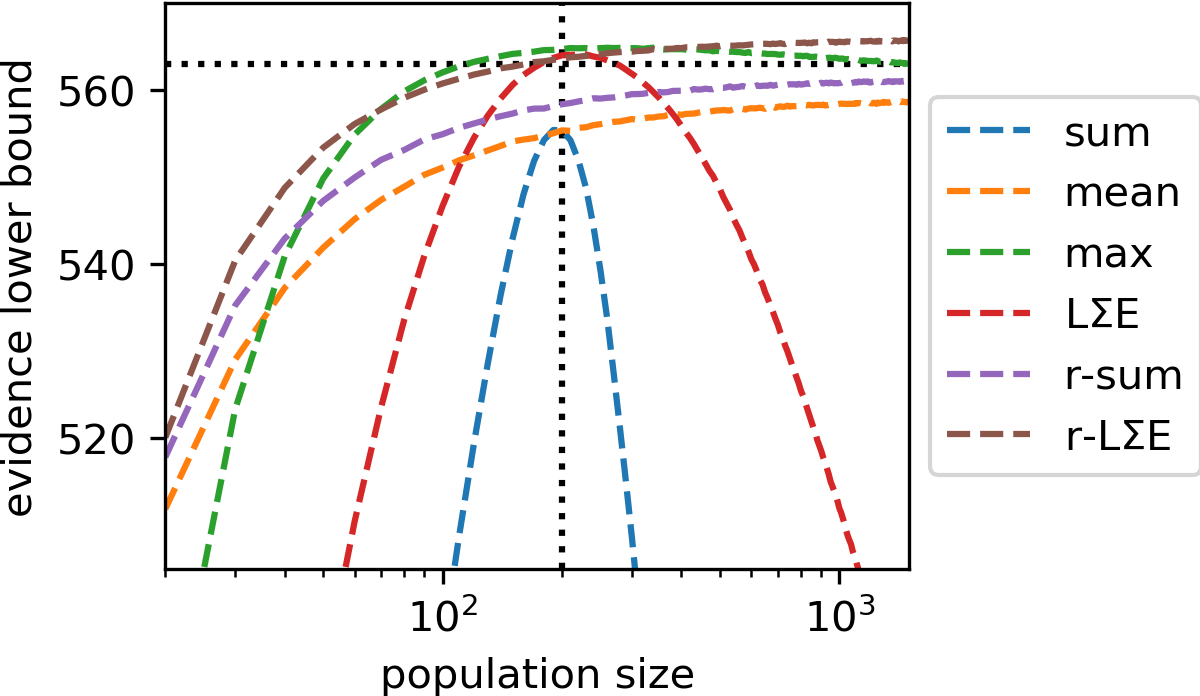}}\hspace{.25cm}
	\caption{Results of the spatial attention experiment.}
	\label{fig:spatialattention}
\end{figure}

In the previous experiments, we investigated models trained in isolation on supervised tasks.
Here, we will test the performance as a building block of a larger model, trained end-to-end and unsupervised.
The data consists of canvases containing multiple MNIST digits, cf.\ \cref{fig:sqair_data}.
In \cite{eslami_attend_2016-1}, an unsupervised algorithm for scene understanding of such canvases was introduced.
We plug an invariant model as the localization module, which repeatedly attends to the input image, at each step returning the bounding box of an object.
To turn a canvas into a population, we interpret the gray-scale image as a two-dimensional density and create populations by sampling 200 particles proportional to the pixel intensities.
Remarkably, the set-based approach requires an order of magnitude fewer weights, and consequently has a significantly lower memory footprint compared to the original model, which repeatedly processes the entire image.

The task is challenging in several ways:
the loss is a lower bound to the likelihood of the input canvas, devoid of localization information.
The intended localization behavior needs to emerge from interaction with downstream components of the overall model.
As with enclosing circles, the bounding box center is correlated with the sample mean of isolated particles from one digit.
However, depending on the digit, this can be inaccurate.

As \cref{fig:air_elbo} indicates, the order-invariant architecture on 200 particles (as in training, vertical line) can serve as a drop-in replacement, performing on a par or slightly improved compared to the original model baseline, indicated by the vertical line.
This is remarkable, with the original model being notoriously hard to train \cite{kosiorek_sequential_2018}.

We investigate the performance of the model when the population size varies.
We observe that the effect on performance varies with different aggregation functions.
Learnable aggregation functions exhibit strictly monotonic performance improvements.
This is reflected by tightening bounding boxes for increasing population sizes, \cref{fig:sqair_data}.
Similar behavior cannot be found reliably for non-learnable aggregations.
Note that we can trade off performance and inference speed \emph{at test time} by varying the population size.

Lastly, we note that in both this and the point cloud experiment, \cref{sub:point_clouds}, learnable $\LSE$-aggregations performed well.
We attribute this to the properties of diminishing returns and sum-$\max$-interpolation amplified by weighted inputs, cf.\ \cref{sec:experiments}. %
\section{Discussion and Conclusion}
We investigated aggregation functions for order-invariant neural architectures.
We discussed alternatives to previously used aggregations.
Introducing recurrent aggregations, we showed that each component of the Deep Set framework can be learnable.
Establishing the notion of sum isomorphism, we created ground for future aggregation models.

Our empirical studies showed that aggregation functions are indeed an orthogonal research axis within the Deep Set framework worth studying.
The right choice of aggregation function may depend on the type of task (\eg regression vs.\ classification).
It affects not only training performance, but also model sensitivity to hyper-parameters and test time performance on out-of-distribution population sizes.
We showed that the learnable aggregation functions introduced in this work are more robust in their performance and more consistent in their estimates with growing population sizes.
Lastly, we showed how to exploit these features in larger architectures by using neural set architectures as drop-in replacements.
In the light of our experimental results, we strongly encourage emphasizing desirable properties of invariant functions, and in particular actively challenge models in non-training scenarios in future research. %

\bibliographystyle{splncs04}
\bibliography{neural_statistics.bib}
\end{document}